\documentclass[sts]{imsart}

\RequirePackage[colorlinks,citecolor=blue,urlcolor=blue,linkcolor=red]{hyperref}
\usepackage{natbib}
\usepackage{graphicx}

\usepackage{amsthm,amsmath,amssymb,amsfonts}
\usepackage{mathtools}
\usepackage{url}
\usepackage{algorithm,algorithmic}
\usepackage{wrapfig}
\usepackage{cleveref}

\newtheorem{proposition}{Proposition}
\newtheorem{remark}{Remark}

\newcommand{\R}{\mathbb{R}}
\newcommand{\X}{X}
\newcommand{\Y}{Y}
\DeclareMathOperator\erf{erf}

\newcommand{\dz}{\textrm{d}z}

\begin{document}

\begin{frontmatter}
        \title{A generative flow for conditional sampling via optimal transport}
        \runtitle{Conditional flow model via optimal transport}
        
        \begin{aug}
        \author{\fnms{Jason}~\snm{Alfonso}},
 		\author{\fnms{Ricardo} \snm{Baptista}\thanks{This paper results from a project in the Polymath Junior undergraduate research program. The project was mentored by RB and GT who designed the project. All the authors contributed to the code and performed the research. RB and GT wrote the paper. Corresponding authors: \texttt{rsb@caltech.edu}, \texttt{giulio.trigila@baruch.cuny.edu}.}},
        \author{\fnms{Anupam}~\snm{Bhakta}},
        \author{\fnms{Noam}~\snm{Gal}},
        \author{\fnms{Alfin}~\snm{Hou}},
        \author{\fnms{Isa}~\snm{Lyubimova}},
        \author{\fnms{Daniel}~\snm{Pocklington}},
        \author{\fnms{Josef}~\snm{Sajonz}},
        \author{
            \fnms{Giulio}~\snm{Trigila\textcolor{red}{$^{\ast}$},} 
        }
        \and
        \author{\fnms{Ryan}~\snm{Tsai}}
        
        \address[ja]{University of the Philippines Diliman\\
        \texttt{jasonbalfonso@gmail.com}}
        
  		\address[rsb]{%
        California Institute of Technology\\
        \texttt{rsb@caltech.edu}}
        
        \address[bc]{%
        Columbia University\\
        \texttt{ab5494@columbia.edu}}
        
        \address[ng]{Baruch College\\\texttt{noamjgal@gmail.com}}
        
        \address[ah]{University of Edinburgh\\\texttt{H.Hou-7@sms.ed.ac.uk}}
        
        \address[il]{Georgia Institute of Technology\\\texttt{vasilisa.constance@gmail.com}}

        \address[dp]{Grinnell College\\\texttt{pockling@grinnell.edu}}

        \address[js]{%
        University of Michigan\\
        \texttt{josefsaj@umich.edu}}
        
        \address[rt]{%
        Yale University\\
        \texttt{ryan.tsai@yale.edu}}

        \address[gt]{%
        Baruch College\\
        \texttt{giulio.trigila@baruch.cuny.edu}}

		\end{aug}
  
	\begin{abstract}
    Sampling conditional distributions is a fundamental task for Bayesian inference and density estimation. Generative models, such as normalizing flows and generative adversarial networks, characterize  conditional distributions by learning a transport map that pushes forward a simple reference (e.g., a standard Gaussian) to a target distribution. While these approaches successfully describe many non-Gaussian problems, their performance is often limited by parametric bias and the reliability of gradient-based (adversarial) optimizers to learn these transformations. This work proposes a non-parametric generative model %
    that iteratively maps reference samples to the target. The model uses block-triangular transport maps, whose components are shown to characterize conditionals of the target distribution. These maps arise from solving an optimal transport problem with a weighted $L^2$ cost function, thereby extending the data-driven approach in~\citet{trigila2016data} for conditional sampling. The proposed approach is %
    demonstrated on a two dimensional example and on a parameter inference problem involving nonlinear ODEs.
    \end{abstract}
		
     
\end{frontmatter}

\section{Introduction} \label{sec:introduction}

Characterizing the conditional distribution of parameters $\X \in \R^d$ in a statistical model given a realization $y^*$ of observations $\Y \in \R^m$ is the fundamental task of computational Bayesian inference. For many statistical models, the posterior measure $\mu(x|y^*) \propto \mu(y^*|x)\mu(x)$ for the likelihood model $\mu(y|x)$ and prior $\mu(x)$ is unknown in closed form and requires either sampling approaches, such as Markov-chain Monte Carlo (MCMC)~\citep{robert1999monte}, or variational methods to approximate the distribution~\citep{blei2017variational}. While MCMC has many consistency guarantees, it is often difficult to produce uncorrelated samples for distributions with high-dimensional parameters and multi-modal behavior.

Generative modeling is a popular framework that
avoids some of the drawbacks associated with MCMC-based sampling methods by making use of transportation of measure%
~\citep{marzouk2016introduction, ruthotto2021introduction, papamakarios2021normalizing, kobyzev2020normalizing}. Broadly speaking, this approach finds a transport map $T$ that pushes forward a reference distribution $\rho$ that is easy to sample (e.g., a standard Gaussian) to the target distribution $\mu$, which we denote as $T_\sharp \rho = \mu$. This map is often found by minimizing the KL divergence using the change-of-variables formula, a technique which first appeared in~\cite{tabak2010density}, or by minimizing Wasserstein distances as in~\cite{arjovsky2017wasserstein}. After finding a transport map $T$, one can immediately generate i.i.d.\thinspace samples in parallel from the target distribution by sampling $z^i \sim \rho$ and evaluating the map at these samples $T(z^i) \sim \mu$, thereby avoiding the use of Markov chain simulation.

In many inference problems, the likelihood model $\mu(y^*|x)$ is computational expensive or intractable to evaluate (e.g., it involves marginalization over a set of high-dimensional latent variables) or the prior density is unavailable (e.g., it is only prescribed empirically by a collection of images). In these settings, evaluating the posterior density of $X|Y = y^*$ up to a normalizing constant, and hence variational inference, is not possible. 
Instead, likelihood-free (which is also known as simulation based) inference~\citep{cranmer2020frontier} aims to sample the posterior distribution given only a collection of samples $(x^i,y^i) \sim \mu(x,y)$ from the joint distribution\footnote{Even if the likelihood and/or prior are intractable, it is often feasible to sample parameters $x^i \sim \mu(x)$ from the prior distribution and synthetic observations $y^i \sim \mu(\cdot|x^i)$ from the likelihood model.}. A broad class of generative models considered in~\citet{spantini2022coupling, kovachki2020conditional, taghvaei2022optimal} for sampling conditional distributions uses transport maps with the lower block-triangular\footnote{We can equivalently consider upper-triangular structure with a reverse ordering for $T^{\mathcal{X}}$ and $T^{\mathcal{Y}}$.} structure %
\begin{equation} \label{eq:block_triangular}
T(y,x) = \begin{bmatrix*}[l] T^\mathcal{Y}(y) \\ T^\mathcal{X}(y,x) \end{bmatrix*},
\end{equation}
where $T^{\mathcal{Y}} \colon \R^m \rightarrow \R^m$ and $T^\mathcal{X} \colon \R^{m+d} \rightarrow \R^d$. In particular, Theorem 2.4 in~\citet{kovachki2020conditional} shows that if the reference density has the product form $\rho(y,x) = \mu(y)\rho(x)$ and $T_\sharp \rho(y,x) = \mu(y,x)$, then $T^\mathcal{X}(y^*,\cdot)_\sharp \rho(x) = \mu(x|y^*)$ for $\mu(y)$-a.e.\thinspace $y^*$. %
Hence, the map $T^\mathcal{X}$ can be used to sample any conditional of the joint distribution. Moreover, we can learn maps of the form in~\eqref{eq:block_triangular} given only samples from the joint distribution~\citep{marzouk2016introduction}. 

Most approaches (see related work below) find such a transport map of the form in~\eqref{eq:block_triangular} by imposing a  parametric form for $T$ and learning its parameters by the solution of a (possibly adversarial) optimization problem~\citep{taghvaei2022optimal, bunne2022supervised}. In addition to the challenges of performing high-dimensional optimization, parametric approaches introduce intrinsic bias and result in models that can't be easily updated in an online data setting. %

\paragraph{Our contribution:} We propose a generative flow model that builds a sequence of maps to push forward the reference $\rho(y,x)$ to the joint target distribution $\mu(y,x)$. The flow results from the composition of simple elementary maps $(T_t),$ %
each with the block-triangular form in~\eqref{eq:block_triangular}. The composition of these maps defines a transformation $T = T_K \circ \dots \circ T_2 \circ T_1$, which satisfies the push-forward condition $T_\sharp \rho(y,x) = \mu(y,x)$.
In this work, we take a product reference distribution for $\rho$ as in~\citet{taghvaei2022optimal} and seek the map from $\rho(y,x) = \mu(y)\mu(x)$ to $\mu(y,x) = \mu(y)\mu(x|y)$ (where with a slight abuse of notation $\mu(y)$ and $\mu(x)$ indicate the, in principle different, marginals of $X$ and $Y$ respectively). 
As a result, the first map component can be taken as $T^{\mathcal{Y}} = \text{Id}(y)$ as it preserves the marginal distribution for the observations $\mu(y)$, while the composition of the second map components $T_t^{\mathcal{X}}(y,\cdot)$ pushes forward the prior distribution $\mu(x)$ to the conditional $\mu(x|y)$ for each $y$. As compared to parametric approaches with fixed model capacity, our algorithm iteratively improves the approximation for the map until the push-forward constraint is met. 

The remainder of this article is organized as follows. Sections~\ref{sec:background}-\ref{sec:flows} show how to learn block-triangular maps maps by maximizing a variational objective arising from an optimal transport problem with a weighted $L^2$ cost. %
We show how to find the corresponding flow by solving only minimization problems, in contrast with other conditional generative models that either solve a min-max optimization problem or require the calculation of the Jacobian matrix for the map (see examples in the next section). Section~\ref{sec:numerical_example} illustrates this flow for solving a Bayesian inference problem where the joint distribution of is defined by a $22$-dimensional parameter and observation space.

\section{Related work} 

\paragraph{Conditional generative models:}  %
Several generative approaches build maps for conditional sampling by directly seeking maps $T^{\mathcal{X}}$ parameterized by the conditioning variables $y$. These models include conditional normalizing flows~\citep{trippe2018conditional, winkler2019learning, lueckmann2019likelihood}, conditional generative adversarial networks~\citep{mirza2014conditional, adler2018deep, liu2021wasserstein}, and conditional diffusion models~\citep{batzolis2021conditional, saharia2022image}. These approaches all require a parameterization for the map, or the score function in the case of diffusion models, which construct a stochastic mapping.  %
A way to overcome a fixed parameterization was proposed in~\cite{tabak2013family} where the first modern version of normalizing flows (NF) appeared. NFs build a map from the target density to a reference density, typically a Gaussian, in a gradual way by composing many elementary maps. Rather than parameterizing the overall map at once, one deals with the more straightforward task of parameterizing simple elementary maps whose composition is supposed to reproduce the overall map. NFs were popularized in~\cite{rezende2015variational} in the context of computer vision, where the elementary maps were chosen to be a combination of relatively simple neural networks and affine transformations with tractable Jacobians in order to use likelihood-based training methods. Recently, many more choices of NFs have been proposed; see~\citet{papamakarios2021normalizing,kobyzev2020normalizing} for reviews on this topic. Among these choices are continuous-time NFs that are often parameterized using neural ODEs~\cite{grathwohl2018ffjord,onken2021ot}. Despite their name, modern NF models select a small number of maps $K = \mathcal{O}(1)$ and jointly learn the composed transformation $T_K \circ \cdots \circ T_1$, thereby making NFs similar to seeking a map with a %
a fixed parametric capacity, rather than a flow.

\paragraph{Monte Carlo methods:} A popular family of nonparametric statistical methods for sampling conditional distributions is approximate Bayesian computation (ABC) ~\citep{sisson2018handbook}. These approaches have been proposed in the setting of intractable likelihood functions. %
To bypass the evaluation of the likelihood function, ABC selects a distance function $d \colon \R^m \times \R^m \rightarrow \R_{+}$ (e.g., the $L^2$ norm) and identifies parameters $x^i$ sampled using Monte Carlo simulation, whose synthetic observations $y^i \sim \mu(\cdot|x^i)$ are close to the true observation $y^*$ up to a small tolerance $\epsilon > 0$, i.e., it rejects parameter samples $x^i$ that do not satisfy $d(y^i,y^*) < \epsilon$. While ABC can be shown to exactly sample the posterior distribution as $\epsilon \rightarrow 0$~\citep{barber2015rate}, the large distances between high-dimensional observations often results in ABC rejecting many samples and producing poor posterior approximations~\citep{nott2018high}. Given that many statistical models are often computationally expensive to simulate, this calls for strategies that don't waste any samples from the joint distribution $\mu(y,x)$. %

\paragraph{Optimal transport:} Among all maps that pushforward one measure to another, optimal transport (OT) select maps that minimize an integrated transportation cost of moving mass~\citep{villani2009optimal}. In recent years, an immense set of computational tools have been developed to find OT maps~\citep{peyre2019computational}. 
For instance,~\citet{cuturi2013sinkhorn} showed that Sinkhorn's algorithm is an efficient procedure for solving a regularized OT problem that computes transport plans between two empirical measures. The plan can be used to estimate an approximate transport map~\citep{pooladian2021entropic}. An alternative approach directly learns a continuous map that can be evaluated at any new input (that is not necessarily in the training dataset) by leveraging the analytical structure of the optimal map for the quadratic cost, which is known as the Brenier map~\citep{brenier1991polar}. In particular,~\citet{makkuva2020optimal} parameterized the map $T$ as the gradient of an input convex neural networks~\citep{amos2017input}. 
The Brenier map transports the samples in a single step and can be found by solving an adversarial optimization problem given only samples of the reference and target measures. The latter approach was extended in~\citet{taghvaei2022optimal} for conditional sampling by imposing the block-triangular structure in~\eqref{eq:block_triangular} on $T$, thereby finding the conditional Brenier map~\cite{carlier2016vector}. The requirement to solve challenging min-max problems in these approaches, however, has inspired alternative methods to find the (conditional) Brenier map that are more stable in high dimensions~\citep{uscidda2023monge}. In this work, we propose a flow-based approach based on OT that only requires the solution of minimization problems, such as those appearing in conditional normalizing flows.

\section{Background on optimal transport} \label{sec:background}

Given two measures $\rho,\mu$\footnote{For ease of exposition we treat these measures as having densities on $\R^n$, but can relax this assumption.} on the same dimensional space $\R^n$, the Monge problem seeks a map $T \colon \R^n \rightarrow \R^n$ that satisfies $T_\sharp \rho = \mu$ and minimizes an integrated transportation cost given in terms of $c \colon \R^{n} \times \R^{n} \rightarrow \R_{+}$. Here we will only consider strictly convex cost functions $c$,  %
such as the quadratic cost $c(z,z') = \frac{1}{2}\|z - z'\|^2$. Then, the optimal transport map is the solution to the Monge problem
\begin{equation} \label{eq:MongeProblem}
    \min_T \left\{\int c(z, T(z))\rho(z)\dz : T_\sharp \rho = \mu \right\},
\end{equation}
over all measurable functions with respect to $\rho$.
To consider measures for which problem~\eqref{eq:MongeProblem} does not admit a solution, it is common to work with a relaxation known as the Kantorovich problem that seeks a coupling, or transport plan, $\gamma \colon \R^n \times \R^n \rightarrow \R_{+}$ with marginals $\rho$ and $\mu$. This relaxation looks for a plan that solves 
\begin{equation}\label{eq:Kant}
\min_{\gamma \in \Pi(\rho,\mu)} \int c(z,z')\gamma(z,z')\dz\dz',
\end{equation}
where $\Pi(\rho,\mu)$ denotes all joint probability distributions that satisfy the constraints
 $
\int \gamma(z,z')\dz'=\rho(z),$ and $\int \gamma(z,z')\dz=\mu(z').$ Problem (\ref{eq:Kant}) is the continuous equivalent of a linear programming problem and, as such, it admits a dual formulation that is useful for our purpose. The dual problem consists of solving the maximization problem %
$$\max_{\varphi,\psi} \int \varphi(z)\rho(z)\dz + \int \psi(z')\mu(z')\dz',$$ 
among potential functions $\varphi \colon \R^n \rightarrow \R^n$ and $\psi \colon \R^n \rightarrow \R^n$ satisfying the constraint $\varphi(z) + \psi(z') \leq c(z,z')$ for all $z,z'$. It can be shown that the solution $(\varphi,\psi)$ of the above problem is given by the conjugate pair %
\begin{align*}
    \varphi(z) &= \psi^c(z) \coloneqq \min_{z'}\{c(z',z)-\psi(z')\} \\ %
    \psi(z') &= \varphi^c(z') \coloneqq \min_{z}\{c(z',z)-\varphi(z)\} %
\end{align*}
where $f^c$ denotes the $c$-transform of $f$. One of the most important results of the dual Kantorovich problem is that, for sufficiently smooth $\rho$ and $\mu$, the solution of the dual problem is equivalent to the solution of the Monge optimal transport problem; in other words, when $\rho$ and $\mu$ are sufficiently regular, the optimal plan $\gamma$ is induced by a one-to-one map $T$. Moreover, one can recover the optimal transport map solving~\eqref{eq:MongeProblem} from the solution of the dual problem for any cost function of the form $c(z,z') = h(z - z')$ with $h$ strictly convex as 
\begin{equation} \label{eq:OptimalMap}
T(z)=z - (\nabla h)^{-1}\nabla \varphi(z).
\end{equation} We refer the reader to~\cite[Chapter 1.3]{santambrogio2015optimal} and~\cite[Chapter 2]{figalli2021invitation} for more details on the solution of the dual formulation for general costs. 

Inspired by the form of the optimizer,~\citet{chartrand2009gradient, gangbo1994elementary} showed that the optimal potentials (and thus the optimal map by~\eqref{eq:OptimalMap}) can be directly computed by maximizing the objective functional %
\begin{equation} \label{eq:dual_objective}
    \mathcal{J}(\varphi) = \int \varphi(z)\rho(z)\dz + \int \varphi^{c}(z')\mu(z')\dz'.
\end{equation}
Moreover, the authors showed that the first variation, i.e., the functional derivative, of the %
objective $\mathcal{J}$ for the quadratic cost $h(z) = \frac{1}{2}\|z\|^2$ at $\varphi_t$ can be explicitly computed as $\textstyle \left. \frac{\delta \mathcal{J}}{\delta \varphi}\right|_{\varphi_t} = \rho(z) - \mu(\nabla \varphi_t^{**})\det \nabla^2 \varphi_t^{**}$, where $\varphi_t^{*}$ denotes the convex conjugate of $\varphi_t$. This suggests that a natural way to solve~\eqref{eq:dual_objective} is via the gradient ascent iterations
\begin{equation}\label{eq:compli}
\varphi_{t+1}(z)=\varphi_{t}(z) + \alpha \left. \frac{\delta \mathcal{J}}{\delta \varphi}\right|_{\varphi_{t}}, %
\end{equation}
where $\alpha > 0$ denotes a step-size parameter. %
Applying this iteration in practice, however, requires the functional form of the source and target densities as well as evaluating convex conjugates via the solution of separate optimization problems. The next section constructs a flow for which we can more easily evaluate the functional derivatives of the objective functional. %

\section{Conditional transport via data-driven flows} \label{sec:flows}

Given that the optimal map is the gradient of the optimal potential $\varphi$, one way to look at the gradient ascent iteration for the potentials is to take the gradient with respect to $z$ on both sides of~\eqref{eq:compli} in order to obtain the discrete-time evolution equation %
\begin{equation}\label{eq:zdynamic}
z_{t+1} = z_{t} - \alpha (\nabla h)^{-1} \nabla_{z}\left. \frac{\delta \mathcal{J}}{\delta \varphi}\right|_{\varphi_{t}}, %
\end{equation}
starting from the identity map $z_{0}=z,$ or equivalently $\varphi_0(z) = \|z\|^2/2$ for the quadratic cost. In the limit of $t\rightarrow \infty$ and $\alpha \rightarrow 0$, the evolution in~\eqref{eq:zdynamic} defines a map $z_{\infty}(z)$ pushing forward $\rho$ to $\mu$. The challenge of considering this dynamic for $\varphi$ is that computing the functional derivative is not straightforward due to the presence of convex conjugates in the definition for $\mathcal{J}$, as in~\eqref{eq:compli}. %

A crucial observation made in~\citet{trigila2016data} shows that one can  substitute~\eqref{eq:zdynamic} with
\begin{equation}\label{eq:simpl}
z_{t+1}=z_{t} - \alpha (\nabla h)^{-1} \nabla_{z} \left. \frac{\delta \mathcal{J}_{t}}{\delta \varphi}\right|_{\varphi = \text{const.}}%
\end{equation}
in terms of the time-dependent functional 
\begin{equation} \label{eq:time_dep_functional}
    \mathcal{J}_{t}(\varphi)=\int \varphi(z)\rho_{t}(z)\dz + \int \varphi^{c}(z')\mu(z)\dz',
\end{equation}
where $\rho_{t}$ is defined as the pushforward of $\rho$ under the map $z_{t}(z)$. In this case, the functional derivative evaluated at a constant potential $\varphi$, that without loss of generality we take to be zero, %
was shown in~\citet{trigila2016data} to be %
$
\left. \frac{\delta \mathcal{J}_t}{\delta \varphi}\right|_{\varphi=0}=\rho_t(z)-\mu(z).
$ This computation avoids the use of convex conjugates as is in Section~\ref{sec:background}. A parametric approximation of this functional derivative will be presented in Section~\ref{sec:parametric_approx}.  %

With $\alpha \rightarrow 0$, the iterations in~\eqref{eq:simpl} define a continuous-time flow gradually mapping $\rho$ into $\mu$. %
The flow evolves according to the dynamic %
$\dot{z} = -(\nabla h)^{-1} \nabla_z (\rho_t(z) - \mu(z))$, and the corresponding probability density function $\rho_t$ for $z_t$ satisfies the continuity equation $\frac{\partial \rho_t}{\partial t} + \text{div}(\rho_t \dot{z}) = 0$. Section 5 in~\cite{trigila2016data} shows that for strictly convex cost functions, the squared $L^2$ norm between $\rho_t$ and $\mu$ is strictly decreasing, which shows that $\rho_t \rightarrow \mu$ in $L^2$ as $t \rightarrow \infty$. An important direction of future work is to establish convergence of the flow under different metrics, and moreover to determine its rates of convergence in relation to properties of the target measure. 

\subsection{Block-triangular maps}
With the quadratic cost $c$, the flow in~\eqref{eq:simpl} does not yield maps with the block-triangular structure in~\eqref{eq:block_triangular} whose blocks can be used for %
conditional sampling. 
To find a block-triangular transport map for $z = (y,x)$, one can use a cost function that heavily penalizes mass movements in the $y$ variable while making almost free movements in the $x$ variable. An example is
\begin{equation} \label{eq:block_cost}
c_\lambda(z,z') = \frac{1}{2}(\lambda\|y - y'\|^2 + \|x - x'\|^2),
\end{equation}
with large positive $\lambda$. %
In this case, the optimal map in~\eqref{eq:OptimalMap} has the form
\begin{equation} \label{eq:map}
T(z)=z - \nabla_{\lambda} \varphi(z),
\end{equation}
where we define the re-scaled gradient associated with the cost $c_{\lambda}$ as $\nabla_{\lambda}\varphi(z) =(\partial_{y}\varphi(z)/\lambda,\partial_{x}\varphi(z))$. %
Hence, when $\lambda \rightarrow \infty$ the map in~\eqref{eq:map} converges to a block-triangular map of the form in~\eqref{eq:block_triangular} with $T^{\mathcal{Y}}(y) = \text{Id}(y)$ and $T^{\mathcal{X}}(y,x) = x + \partial_x \varphi(y,x)$. 

\begin{remark} \label{rem:conditional_Brenier}
    The optimal transport map $T$ pushing forward $\rho(x)$ to $\mu(x|y)$ with minimal quadratic cost $\int \|x - T(y,x)\|^2 \rho(y,x) \textrm{d} x \textrm{d} y$ for transporting the $x$ variable in expectation over $y \sim \mu(y)$ was coined in~\citet{carlier2016vector} as the conditional Brenier map. Theorem 2.3 in~\citet{carlier2016vector} shows that this map is monotone and unique among all functions written as the gradient of a convex potential with respect to the input $x$. %
\end{remark}

\begin{remark} The cost function in~\eqref{eq:block_cost} is related to the weighted $L^2$ cost function $\sum_{i=1}^n \lambda_i(\varepsilon) |z_i - z_i'|^2$ for $\lambda_i(\varepsilon) > 0$. For weights satisfying $\lambda_{i+1}(\varepsilon)/\lambda_i(\varepsilon) \rightarrow 0$ as $\varepsilon \rightarrow 0$ for all $i \in \{1,\dots,d-1\}$, \citet{carlier2010knothe} showed that the optimal transport map with respect to this weighted cost converges to the strictly lower-triangular transport map known as the Knothe-Rosenblatt (KR) rearrangement~\citep{knothe1957contributions, rosenblatt1952remarks}. The KR map is uniquely defined given a variable ordering. For the purpose of conditional sampling, it is sufficient to consider block-triangular, rather than triangular, maps, as described in Section~\ref{sec:introduction}. The drawback is that the larger space of block-triangular maps $T$ admits more transformations satisfying the push-forward condition $T_\sharp \rho = \mu$. The non-uniqueness can be resolved, however, by the regularization from the transport cost; see Remark~\ref{rem:conditional_Brenier}. %
\end{remark}

As in the previous section, we  now derive a flow where each elementary map has a block-triangular structure of the form in~\eqref{eq:block_triangular}. By keeping in mind the connection between the map $z_{t+1}(z)$ and the rescaled gradient of a corresponding potential $\varphi_{t+1}$, one can construct a flow of the form in~\eqref{eq:simpl} as $z_{t+1} = z_t - \alpha \nabla_\lambda (\rho_t(z_t) - \mu(z_t))$. %
Each elementary map %
is %
the sum of the identity function 
and a perturbation given by the rescaled gradient of the maximum ascent direction for $\mathcal{J}_t$ at $\rho_t$. Moreover, the functional derivative of $\mathcal{J}_t$ can be computed entirely from the current reference measure $\rho_t$ and the target $\mu$. %

Each elementary map pushes forward $\rho_{t}$ to $\rho_{t+1}$, which approaches $\mu$ as $t \rightarrow \infty$. Moreover, the composition of block-triangular maps $T_t$ (corresponding to $\lambda \rightarrow \infty$) at each step $t$ yields an overall map $T = T_K \circ \cdots \circ T_1$ that is also block-triangular and can be used to sample the conditional distributions $X | Y = y$ for any observation $y \in \R^m$, as will be described in the following section. The next section shows how to compute the functional gradient and the resulting potential at each step given only samples from the reference and target measure. 

\subsection{Gradient approximation from samples} \label{sec:parametric_approx} %

In this work, we follow~\citet{trigila2016data} and approximate the gradient in the span of a small set of features $F_j \colon \R^n \rightarrow \R$ where $n =d+m$ with coefficients $\beta_j \in \R$, i.e., 
\begin{equation}\label{eq:param}
\left. \frac{\delta \mathcal{J}_{t}}{\delta \varphi}\right|_{\varphi(z)=0} \approx \sum_j \beta_j F_j(z).
\end{equation}
The features can include local radial basis functions, polynomials, or possibly neural networks; see~\citet{gu2022lipschitz} for an example in the context of training generative adversarial networks. %
In this work, we chose $F_{j}$ to be radial basis functions centered around a subset of randomly selected points. More details on the parameterization and our algorithm for selecting the centers is provided in Appendix~\ref{sec:parameterization}.  

The approximation in (\ref{eq:param}) corresponds to the parameterization of the potential  
$
\varphi(z)\approx \varphi_{\beta}(z)= \sum_{j}\beta_{j}F_{j}(z).
$
Given a rich expansion for $\varphi$, one can hope to approximate the functional derivative sufficiently well. Nevertheless, a core advantage of the flow is that the elementary map at each step does not need to learn the full map pushing forward $\rho_t$ to $\mu$.

In an empirical setting, our goal is to estimate the potential functions $\varphi_\beta$ at  step $t$ and derive the corresponding flow given only i.i.d.\thinspace samples $\{z_t^i\}_{i=1}^N \sim \rho_t$ and $\{(z')^i\}_{i=1}^M \sim \mu$. In practice, samples from the initial product reference $\rho_0(y,x) = \mu(y)\mu(x)$ can be generated by creating a tensor product set of the joint samples from $\mu(y,x)$.  %
We use the samples to define a Monte Carlo approximation of the objective functional in~\eqref{eq:time_dep_functional}, which is linear with respect to each measure. 
That is, 
\begin{equation} \label{eq:emp_objective}
    \widehat{\mathcal{J}}_t(\varphi_\beta) = \frac{1}{N} \sum_{i=1}^{N} \varphi_\beta(z_{t}^{i}) + \frac{1}{M} \sum_{i=1}^{M} \varphi_\beta^{c}((z')^{i}).
\end{equation}

In this work, we %
select the values of the coefficients $\beta_{j}$ that maximize the objective in~\eqref{eq:emp_objective} via gradient ascent. Instead of selecting a step-size, however, we choose the coefficients according to the one-step Newton scheme $\beta^* = (\nabla_\beta^2 \widehat{\mathcal{J}}_t)^{-1} \nabla_\beta \widehat{\mathcal{J}}_t$, where the first and second derivatives are computed at $\beta = 0$. The Newton method better captures the local curvature of $\mathcal{J}_t$ around the identity map and yields faster convergence of the flow towards $\mu$.
Appendix~\ref{sec:GradientHessian} shows that the gradient and Hessian of~\eqref{eq:emp_objective} with respect to $\beta$ can be easily computed as
\begin{align*}
\left. \nabla_{\beta_j} \widehat{\mathcal{J}}_t \right|_{\beta = 0} &=\frac{1}{N}\sum_{i=1}^{N}F_j(z_t^{i})- \frac{1}{M}\sum_{i=1}^{M}F_j((z')^{i}), \\
\left. \nabla^2_{\beta_j,\beta_k} \widehat{\mathcal{J}}_t  \right|_{\beta = 0} &=-\frac{1}{M}\sum_{i=1}^M \frac{1}{2}\langle \nabla F_j((z')^{i}) ,\nabla_{\lambda} F_k((z')^{i}) \rangle.
\end{align*}

\begin{remark} The expressions for the derivatives of $\widehat{\mathcal{J}}_t$ with respect to $\beta$ are relatively simple because they are computed around $\beta = 0$. This is the analogous to having a closed-form expression for the first variation of $\mathcal{J}_t$ around $\varphi = 0$ as in Section 4.1.
\end{remark}

Once the optimal coefficients $\beta^*$ have been computed, we take the rescaled gradient of the potential $\varphi_{\beta^*}$ to obtain a discrete-time update for the parameterized version of equation~\eqref{eq:simpl}. Each update defines ones elementary map $T_t$ with the block-triangular structure in~\eqref{eq:block_triangular}. That is,
\begin{equation} \label{eq:final_update}
z_{t+1} = T_t(z_t) \coloneqq z_{t} - \sum_{j}\beta^*_{j}\nabla_{\lambda}F_{j}(z_t). 
\end{equation} %

We propose to update the samples until their values are no longer changing; more specifically, our termination criteria compares the update to all sample points $\{z_t\}_{i=1}^N$ in the evolution equation~\eqref{eq:final_update} %
with a threshold value $\epsilon=10^{-6}$. When the samples stop moving they are approximately equal in distribution to the target samples from $\mu$. The  composition of the resulting elementary maps in~\eqref{eq:final_update} define a generative flow model pushing forward $\rho$ to $\mu$. Our complete procedure for learning the flow is provided in Algorithm~\ref{alg:mf}.

\begin{algorithm}[hbt!]
\caption{Generative flow model for conditional sampling}\label{alg:mf}
\begin{algorithmic}[1]
{\normalsize \STATE {\bfseries Input}: Joint samples $\{(y^i,x^i)\}_{i=1}^N \sim \mu(y,x)$, features $(F_j)$, termination threshold $\epsilon$
\STATE Split dataset to create reference $\rho(y,x) = \mu(y)\mu(x)$ and target $\mu(y,x)$ samples
\STATE Set $t = 0$ and $z_t^i = (y_t^i,x_t^i)$ to reference samples
\WHILE{samples are still moving: $\|z_{t+1}^i - z_t^i\| > \epsilon$ for any $i$}
\STATE Find coefficients $\beta^*$ using one-step of Newton’s method
\STATE Move points using the map in~\eqref{eq:final_update} %
\STATE Increment counter $t \leftarrow t+1$
\ENDWHILE}
\end{algorithmic}
\end{algorithm}

The resulting flow defines an overall block-triangular map that can be used to sample any conditional of the target measure. By preserving the block-triangular structure in each component, the composed map after running $K$ steps of Algorithm~\ref{eq:block_triangular} has the block-triangular form in~\eqref{eq:block_triangular} where the second component is given by $T^{\mathcal{X}}(y^*,x) \coloneqq T_K^{\mathcal{X}}(y^*,\cdot) \circ \cdots \circ T_{1}^{\mathcal{X}}(y^*,x)$ for any conditioning variable $y^*$. Theorem 2.4 in~\citet{kovachki2020conditional} shows that the map $x \mapsto T^{\mathcal{X}}(y^*,x)$ pushes forward $\mu(x)$ to $\mu(x|y^*)$. Thus, we can sample any conditional measure after learning the flow by pushing forward (new) prior samples through the composed map with a fixed argument for the $y$ variable. %

While the maps found using Algorithm~\ref{alg:mf} can be used to sample any conditional distribution, we suggest adapting the flow when one is interested in the conditional  distribution corresponding to one realization of the conditioning variable $y^*$. In particular, we can include markers $\{(x^i,y^*)\}_{i=1}^N$ with $x^i \sim \mu(x)$ in the set of reference $\rho$ samples. The push-forward of these additional samples immediately provides samples from the desired conditional distribution $\mu(x|y^*)$. Moreover,
we can select local features $F_j$ centered near $y^*$ in the algorithm. This ensures the flow characterizes finer details of the map in this region of the input space. %
This is analogous to conditional normalizing flow (NF) models that sequentially retrain with $x^i$ samples drawn from an approximation to $\mu(x|y^*)$, instead of the prior $\mu(x)$, in order to more accurately model a specific posterior~\citep{greenberg2019automatic}. Unlike the latter approaches, however, our maps can still be applied to other conditioning variables $y^*$.

We conclude this section by presenting a few core advantages of our flow-based algorithm. First, the algorithm does not require an arbitrary \textit{a-priori} selection of the number of elementary maps (i.e., layers) in the flow, as compared to modern NFs~\citep{papamakarios2021normalizing}. Instead, the algorithm proceeds until the difference between the reference and target measures is small according to the selected features, which can be chosen adaptively at each step. Second, the algorithm only uses simple minimization steps with respect to the parameters as compared to approaches that solve min-max problems. In fact, the complexity of each iteration is at most $\mathcal{O}(Mdp^2 + p^3)$ to update the coefficients, where $p$ is the number of features, and $\mathcal{O}(Ndp)$ to move the sample points. Third, we don't need to evaluate the push-forward density through the map $T$ for training, unlike approaches that use the change-of-variables formula to maximize the likelihood of the data. Hence, the algorithm does not require evaluating the Jacobian matrix of $T$ or require specific parameterizations that guarantee $T$ is invertible and $\det \nabla T$ is tractable to evaluate, as in~\citet{wehenkel2019unconstrained, baptista2020representation}. Lastly, we don't require the functional form of the reference density, which allows one to construct maps that push-forward a general (possibly non-Gaussian) prior measure. %

\section{Numerical examples}  \label{sec:numerical_example}

In this section, we will %
illustrate the flow on a small two-dimensional example in Section~\ref{sec:banana}, and a Bayesian inference problem where we infer four parameters of the Lotka--Volterra nonlinear ODE model in Section~\ref{sec:LotkaVolterra}. 

\subsection{2D banana distribution} \label{sec:banana}
Here, we let the parameter be a standard Gaussian random variable $X\sim \mathcal{N}(0,1)$ and the measurement be $Y=0.5X^{2}-1 + \epsilon$ with $\epsilon\sim \mathcal{N}(0,1)$. The left panel of Figure \ref{fig:SamplesMoon} shows the (un-normalized) joint density $\mu(x,y)$ while the middle panel shows samples from $\mu(x,y)$, in red, and the product reference $\mu(x)\mu(y)$ in blue. In this example we parameterize the elementary maps as perturbations of the identity map where the perturbation is given by the linear combinations of ten radial basis functions with centers chosen at random from $\mu(x)\mu(y)$. The code to reproduce the results for this example is provided at \url{github.com/GiulioTrigila/MoonExample}.

\begin{figure}[!ht]
    \centering
    \includegraphics[trim={1cm 0.5cm 1cm 1cm},clip,width=0.32\textwidth]{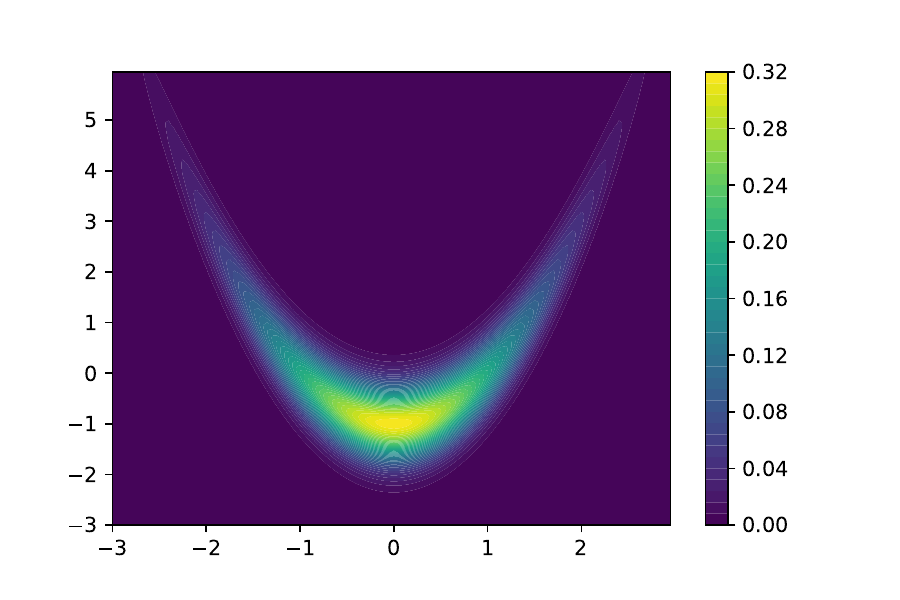}
    \includegraphics[trim={1cm 0.5cm 1cm 1cm},clip,width=0.32\textwidth]{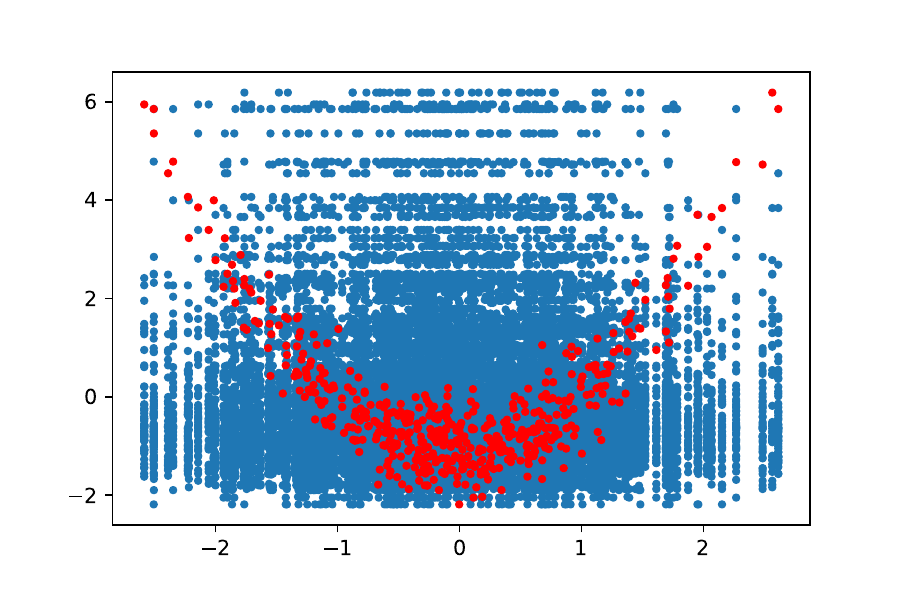}
    \includegraphics[trim={1cm 0.7cm 1cm 1.5cm},clip,width=0.28\textwidth]{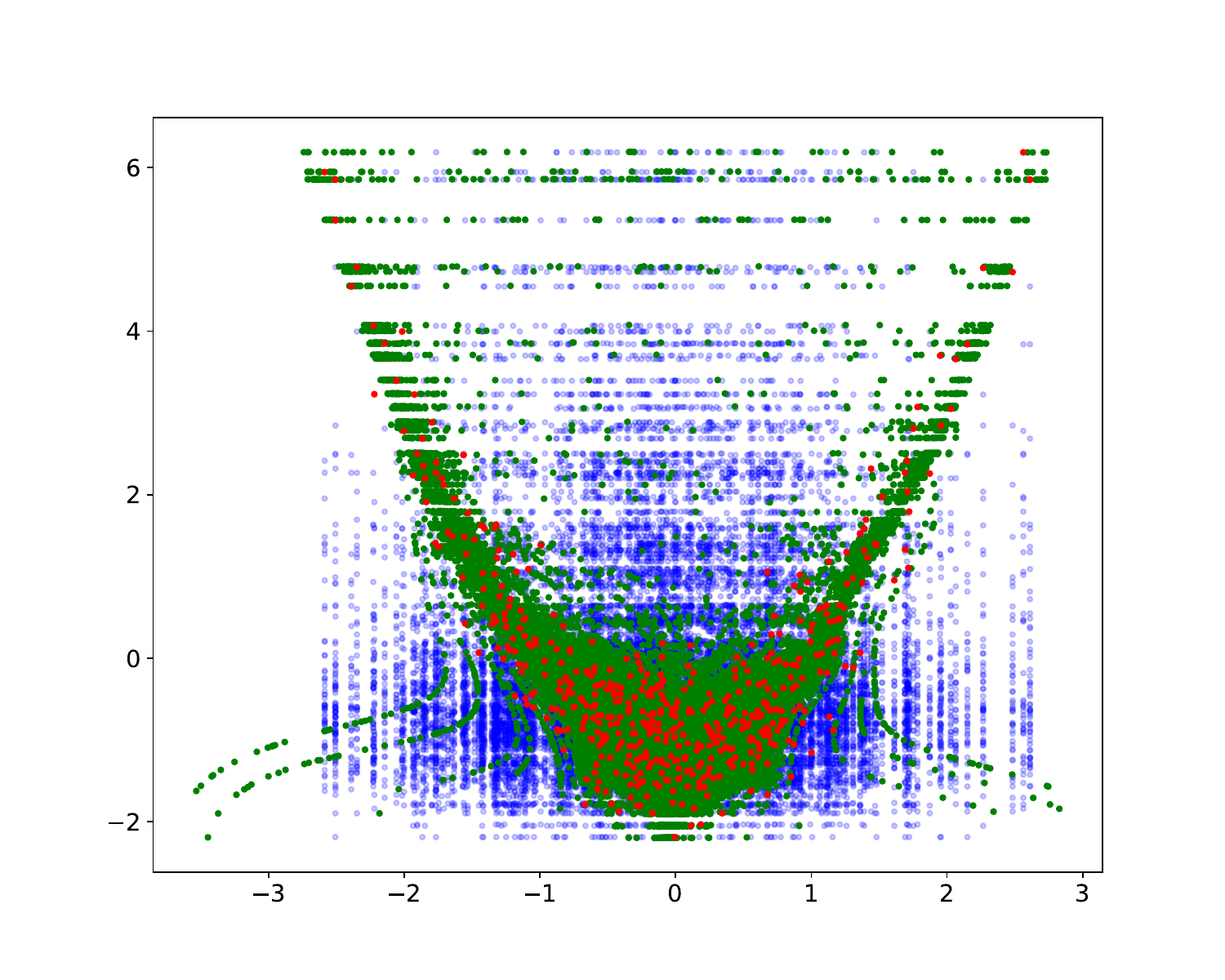}    
    \vspace{-0.2cm}
    \caption{Left: Contours of the joint density $\mu(x,y)$. Middle: $500$ i.i.d.\thinspace samples from $\mu(x,y)$ in red and $10^4$ samples from $\mu(x)\mu(y)$. Right: Samples from $\mu(x,y)$ in red, from $\mu(x)\mu(y)$ in blue, and from the pushforward of $\mu(x)\mu(y)$ through the flow in green. As expected, the pushforward samples overlap with the joint samples.} %
    \label{fig:SamplesMoon}
\end{figure}

The right panel of Figure~\ref{fig:SamplesMoon} plots the samples generated by pushing forward the product reference samples through the composed map $T$. At the end of the algorithm, the push-forward condition $T_\sharp \rho = \mu$ should be satisfied. As a result, the push-forward should match the joint samples, as seen with the close match between the green and red samples. By Theorem 2.4 in~\cite{kovachki2020conditional}, we can then use these maps to sample the conditional distribution $\mu(x|y^\ast)$ for any $y^*$. Figure~\ref{fig:CDE} plots the approximate density (using a kernel density estimator) of the push-forward samples $T(y^*,x^i)$ with $x^i \sim \mu(x)$ for the conditioning variable $y^* = 2$. In comparison to the conditionals of a kernel density estimator of the joint samples from $\mu(y,x)$ or $T_\sharp \rho(y,x)$, we observe close agreement with the true multi-modal conditional density for $\mu(x|y^*)$. We note that for invertible  mappings $T$, the change-of-variables formula $\rho_T(x|y^\ast) \coloneqq \rho(T^{-1}(y^\ast,x))|\nabla T^{-1}(y^\ast,x)|$ can also be used to approximate the conditional density after learning the flow.  %

\begin{figure}[!ht]
    \centering
    \includegraphics[trim={3cm 1.5cm 3cm 1.5cm}, clip, width=0.9\textwidth]{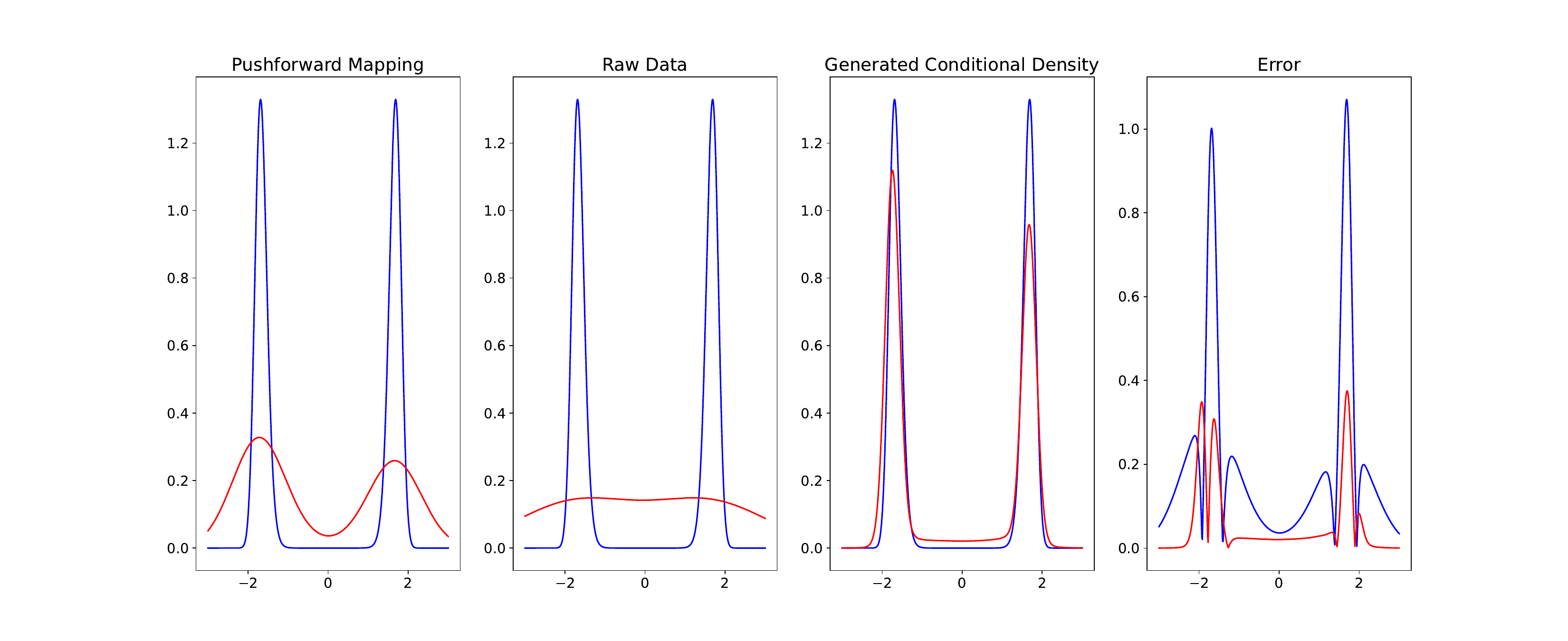}    
    \vspace{-0.2cm}
    \caption{The blue line in all panels represents the true conditional density for $\mu(x|y^*=2)$. The red line in the first (left) panel displays a Nadaraya-Watson \citep{nadaraya1964estimating,watson1964smooth} conditional kernel density estimator (CKDE) of $\mu(x|y^*=2)$ using the samples obtained from the pushforward of $\mu(x)\mu(y)$ (i.e., the green points in Figure~\ref{fig:SamplesMoon}). Similarly, the second panel displays the CKDE of $\mu(x|y^*=2)$ using samples from the joint distribution (i.e., the red points in Figure~\ref{fig:SamplesMoon}). This is similar to a posterior approximation obtained using an Approximate Bayesian Computation (ABC) rejection-sampling algorithm. The third panel displays the (CKDE) of $\mu(x|y^*=2)$ obtained using $10^{4}$ samples generated from $\mu(x)$ and mapped to $\mu(x|y^*=2)$ via the learned %
    flow. The fourth (right) panel compares the error between the true conditional density and the CKDE of the first panel in blue and the CKDE of the third panel in red (i.e., the method proposed in this work). We observe that the flow provides more accurate conditional approximations than ABC given the same joint samples.}
    \label{fig:CDE}
\end{figure}

\subsection{Lotka--Volterra dynamical system} \label{sec:LotkaVolterra}
In this section we apply Algorithm \ref{alg:mf} to estimate static parameters in the Lotka--Volterra population model given noisy realizations of the states over time. The model describes the populations $p = (p_1,p_2) \in \R_{+}^2$ of prey and predator species, respectively. The populations $p(t)$ for times $t \in [0,T]$ solve the nonlinear coupled ODEs
\begin{equation} \label{eq:LV_ODE}
\frac{dp_{1}(t)}{dt}=\alpha p_{1}(t) - \beta p_{1}(t)p_2(t), \qquad 
\frac{dp_{2}(t)}{dt}=-\gamma p_{2}(t) - \delta p_{1}(t)p_{2}(t),
\end{equation}
with the initial conditional $p(0) = (30,1)$, where $X = (\alpha, \beta, \gamma, \delta) \in \R^4$ 
are unknown parameters. 
The parameters are initially distributed according to a log-normal prior distribution given by $\log(X)\sim\mathcal{N}(\mu,0.5I_4)$ with $\mu=(-0.125,-0.125,-3,-3)$. 
We simulate the ODE for $T=20$ time units and observe the state values every $\Delta t_{\text{obs}}=2$ time units with 
independent and additive log-normal noise, i.e., $\log(Y_k)\sim\mathcal{N}(p(k\Delta t_{\text{obs}}),\sigma^{2}I_{2})$ for $k=1,\dots,9$ with $\sigma^2=0.1$.
Figure~\ref{fig:LVobser} presents the two states $p(t)$ in solid lines for the parameter $x^{\ast}=(0.83, 0.041, 1.08 , 0.04)$ %
and an observation $y^{\ast}\in\R^{18}$ drawn from the likelihood model $\mu(\cdot|x^*)$ in circles. %
The main reason for choosing this model to test the procedure described in Section~\ref{sec:flows} is that the likelihood model is known in closed form and thus the results can be compared to MCMC sampling procedures, the gold standard for Bayesian inference methods. In this experiment we learned the flow using only $M = 1000$ samples from the joint distribution.%

\begin{figure}[!ht]
    \vspace{-0.5em}
    \centering
    \includegraphics[trim={1cm 0.5cm 1cm 0.5cm},clip,width=0.38\textwidth]{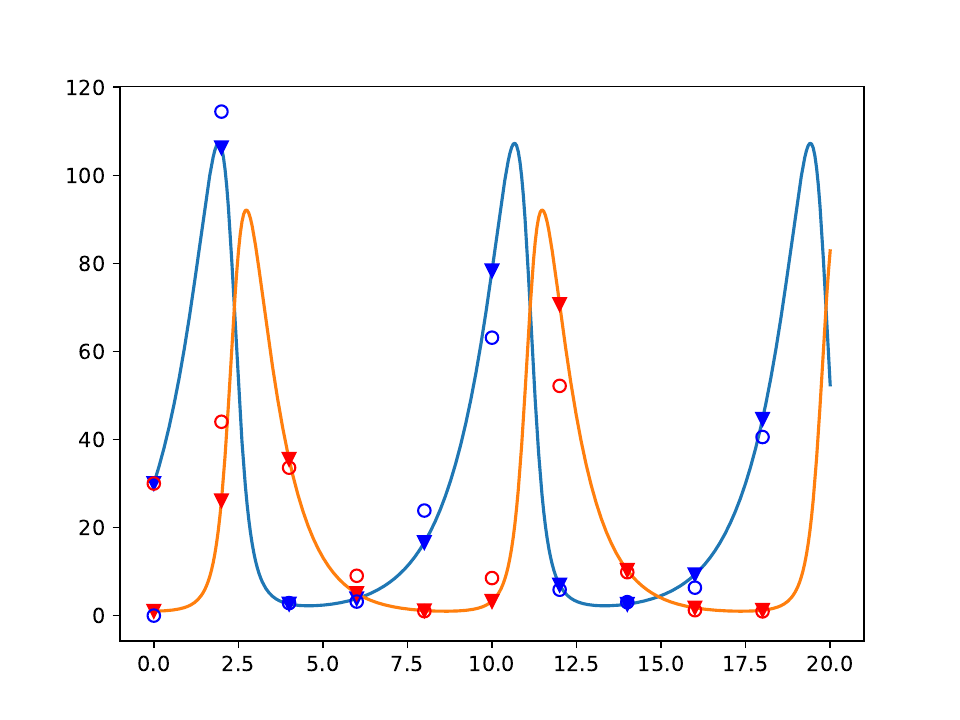}
    \vspace{-0.8em}
    \caption{Populations of the two species ($y$-axis, in blue and orange) as a function of time for the parameters $x^{*}$. The triangles are observations from the true trajectory and the circles are the observations $y^*$ with additive noise.}
    \label{fig:LVobser}
\end{figure}

Figure~\ref{fig:LVPost} compares one and two-dimensional projections of the approximate posterior distribution for the parameters obtained with our procedure (left) and with MCMC (on the right). The red vertical line represents the exact value of the parameters $x^\ast$ used to generate the trajectory in Figure~\ref{fig:LVobser}. %
For each approximation, we obtain the 30 most significant posterior samples, which are closest to the empirical posterior mean in the Euclidean norm. For each parameter, we find the corresponding state trajectories by solving the ODEs in~\eqref{eq:LV_ODE}. The states obtained with the flow and with MCMC are compared in the left and right of Figure~\ref{fig:trajectories} respectively.
The posterior predictive states give a visual representation of the uncertainty arising from estimating the true parameter $x^{\ast}$ given noisy observations. 
As expected, the MCMC method displays lower uncertainty in the trajectories due to its use of the exact likelihood model. This is particularly noticeable for larger values of the populations $p_1,p_2$ where the effect of the noise on the sampled data (i.e., the difference between the circle and triangle markers in Figure~\ref{fig:LVobser}) has a larger effect than during time intervals where $p_{1}$ and $p_{2}$ are nearly constant. 

\begin{figure}[!ht]
    \centering
    \includegraphics[width=0.45\textwidth]{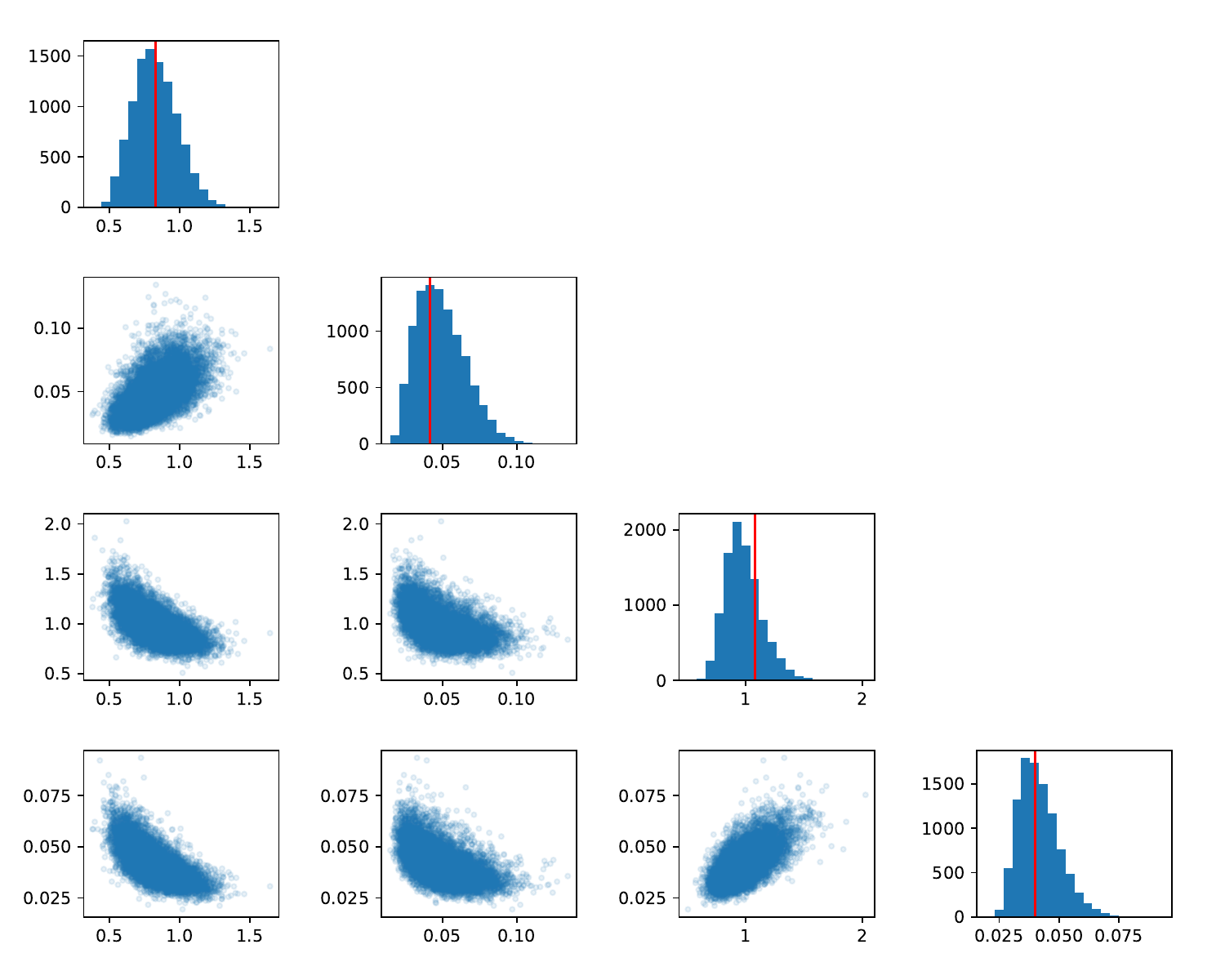}
    \includegraphics[width=0.45\textwidth]{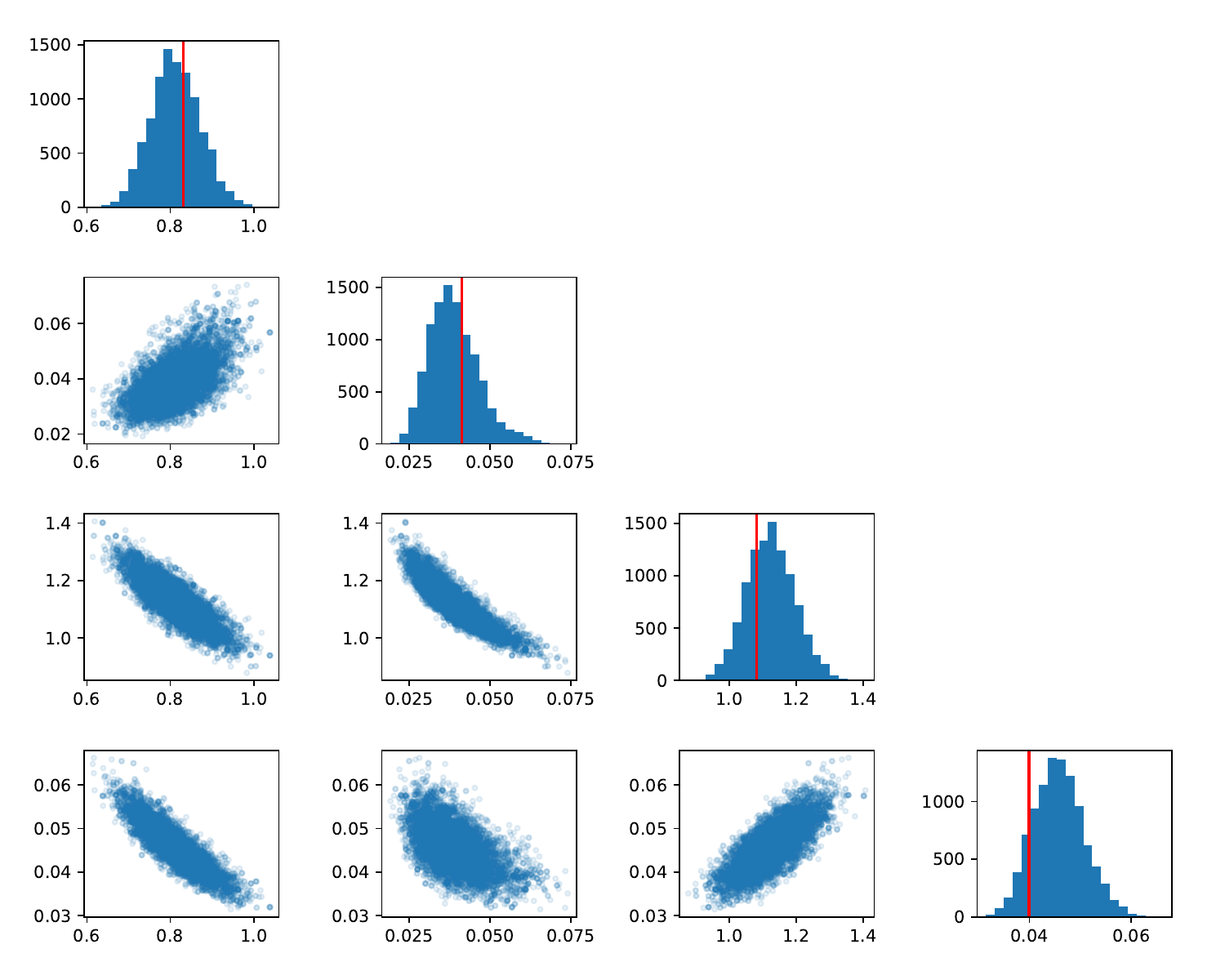}
    \vspace{-0.3cm}
    \caption{Left: approximate posterior samples generated by mapping $10^4$ points from the prior using the map computed from the flow. Right: $10^4$ MCMC samples. Both simulations are compared to the true parameters $x^\ast$ (in red) that generated the observations $y^\ast$ in Figure~\ref{fig:LVobser}. The panels on the main diagonal represent the histogram for each of the four parameters. The off-diagonal panels represent two dimensional projections of the samples drawn from the posterior. Panels on the same column share the same $x$ axis relative to the value of parameter.} %
    \label{fig:LVPost}
\end{figure}

\begin{figure}[!ht]
    \centering
    \vspace{-0.5cm}
    \includegraphics[width=0.45\textwidth]{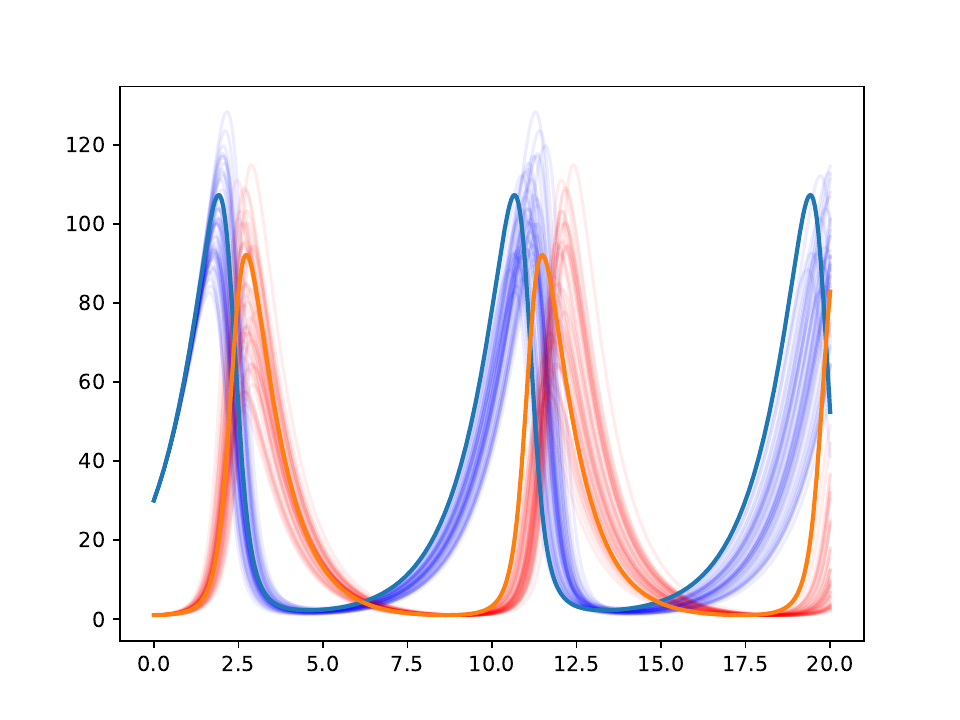}
    \includegraphics[width=0.51\textwidth]{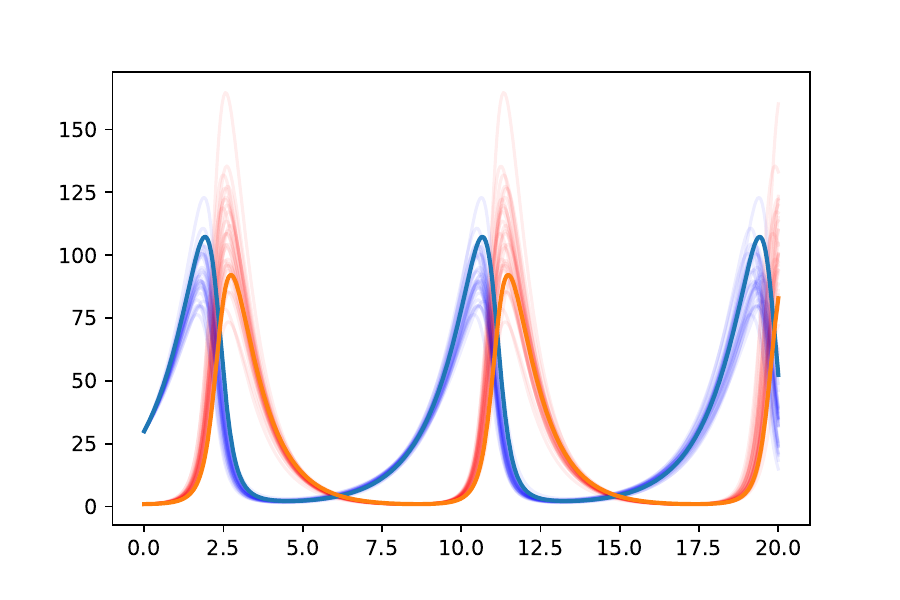}
    \vspace{-0.5cm}
    \caption{Populations as a function of time obtained by solving the ODE model with the parameter posterior samples displayed in Figure \ref{fig:LVPost} from the flow model (left) and MCMC (right). The parameters were chosen to be the 30 closest values (in the Euclidean norm) to the empirical posterior mean.}
    \label{fig:trajectories}
\end{figure}

\section{Conclusions and future work}
This work presents a generative flow model for Bayesian inference where samples of the posterior distribution are generated by pushing forward prior samples through a composition of maps. Finding the flow is entirely data-driven and it is based on the theory of optimal transport (OT) with a weighted $L^2$ cost function. This cost %
yields transport maps with a block-triangular structure, which are are suitable for conditional sampling. The algorithm's performance is illustrated on a two dimensional synthetic example and a Bayesian inference problem involving nonlinear ODEs. %
The advantage of this approach over state-of-the-art algorithms that seek OT maps for conditional sampling is that it only relies on pure minimization steps at each step of the flow rather than solving more challenging min-max optimization problems. The flow is constructed through elementary maps that, acting locally around where the samples are concentrated, gradually push forward the prior to the posterior distribution. Future work includes the possibility of enlarging the feature space for each map by means of projections into %
reproducing Kernel Hilbert spaces in a similar spirit to~\cite{chewi2020svgd}, as well as adapting the features to exploit low-dimensional structure between the reference and target distributions as in~\cite{brennan2020greedy}.

\section*{Acknowledgements}
This work was done as part of the 2022 Polymath Junior summer undergraduate research program, supported by the NSF REU program under award DMS-2218374. RB gratefully acknowledges support from the US Department of Energy AEOLUS center (award DE-SC0019303), the Air Force Office of Scientific Research MURI on ``Machine Learning and Physics-Based Modeling and Simulation'' (award FA9550-20-1-0358), and a Department of Defense (DoD) Vannevar Bush Faculty Fellowship (award N00014-22-1-2790). 

\bibliographystyle{plainnat}
\bibliography{references}

\begin{thebibliography}{53}
\providecommand{\natexlab}[1]{#1}
\providecommand{\url}[1]{\texttt{#1}}
\expandafter\ifx\csname urlstyle\endcsname\relax
  \providecommand{\doi}[1]{doi: #1}\else
  \providecommand{\doi}{doi: \begingroup \urlstyle{rm}\Url}\fi

\bibitem[Adler and {\"O}ktem(2018)]{adler2018deep}
J.~Adler and O.~{\"O}ktem.
\newblock Deep bayesian inversion.
\newblock \emph{arXiv preprint arXiv:1811.05910}, 2018.

\bibitem[Amos et~al.(2017)Amos, Xu, and Kolter]{amos2017input}
B.~Amos, L.~Xu, and J.~Z. Kolter.
\newblock Input convex neural networks.
\newblock In \emph{International Conference on Machine Learning}, pages
  146--155. PMLR, 2017.

\bibitem[Arjovsky et~al.(2017)Arjovsky, Chintala, and
  Bottou]{arjovsky2017wasserstein}
M.~Arjovsky, S.~Chintala, and L.~Bottou.
\newblock Wasserstein generative adversarial networks.
\newblock In \emph{International conference on machine learning}, pages
  214--223. PMLR, 2017.

\bibitem[Baptista et~al.(2020)Baptista, Marzouk, and
  Zahm]{baptista2020representation}
R.~Baptista, Y.~Marzouk, and O.~Zahm.
\newblock On the representation and learning of monotone triangular transport
  maps.
\newblock \emph{arXiv preprint arXiv:2009.10303}, 2020.

\bibitem[Baptista et~al.(2023)Baptista, Hosseini, Kovachki, and
  Marzouk]{kovachki2020conditional}
R.~Baptista, B.~Hosseini, N.~B. Kovachki, and Y.~Marzouk.
\newblock Conditional sampling with monotone {GAN}s: from generative models to
  likelihood-free inference.
\newblock \emph{arXiv preprint arXiv:2006.06755}, 2023.

\bibitem[Barber et~al.(2015)Barber, Voss, and Webster]{barber2015rate}
S.~Barber, J.~Voss, and M.~Webster.
\newblock {The rate of convergence for approximate {B}ayesian computation}.
\newblock \emph{Electronic Journal of Statistics}, 9\penalty0 (1):\penalty0 80
  -- 105, 2015.

\bibitem[Batzolis et~al.(2021)Batzolis, Stanczuk, Sch{\"o}nlieb, and
  Etmann]{batzolis2021conditional}
G.~Batzolis, J.~Stanczuk, C.-B. Sch{\"o}nlieb, and C.~Etmann.
\newblock Conditional image generation with score-based diffusion models.
\newblock \emph{arXiv preprint arXiv:2111.13606}, 2021.

\bibitem[Blei et~al.(2017)Blei, Kucukelbir, and McAuliffe]{blei2017variational}
D.~M. Blei, A.~Kucukelbir, and J.~D. McAuliffe.
\newblock Variational inference: A review for statisticians.
\newblock \emph{Journal of the American statistical Association}, 112\penalty0
  (518):\penalty0 859--877, 2017.

\bibitem[Brenier(1991)]{brenier1991polar}
Y.~Brenier.
\newblock Polar factorization and monotone rearrangement of vector-valued
  functions.
\newblock \emph{Communications on pure and applied mathematics}, 44\penalty0
  (4):\penalty0 375--417, 1991.

\bibitem[Brennan et~al.(2020)Brennan, Bigoni, Zahm, Spantini, and
  Marzouk]{brennan2020greedy}
M.~Brennan, D.~Bigoni, O.~Zahm, A.~Spantini, and Y.~Marzouk.
\newblock Greedy inference with structure-exploiting lazy maps.
\newblock \emph{Advances in Neural Information Processing Systems},
  33:\penalty0 8330--8342, 2020.

\bibitem[Bunne et~al.(2022)Bunne, Krause, and Cuturi]{bunne2022supervised}
C.~Bunne, A.~Krause, and M.~Cuturi.
\newblock Supervised training of conditional monge maps.
\newblock \emph{Advances in Neural Information Processing Systems},
  35:\penalty0 6859--6872, 2022.

\bibitem[Carlier et~al.(2010)Carlier, Galichon, and
  Santambrogio]{carlier2010knothe}
G.~Carlier, A.~Galichon, and F.~Santambrogio.
\newblock From knothe's transport to brenier's map and a continuation method
  for optimal transport.
\newblock \emph{SIAM Journal on Mathematical Analysis}, 41\penalty0
  (6):\penalty0 2554--2576, 2010.

\bibitem[Carlier et~al.(2016)Carlier, Chernozhukov, and
  Galichon]{carlier2016vector}
G.~Carlier, V.~Chernozhukov, and A.~Galichon.
\newblock Vector quantile regression: An optimal transport approach.
\newblock \emph{Annals of Statistics}, 44\penalty0 (3):\penalty0 1165--1192,
  2016.

\bibitem[Chartrand et~al.(2009)Chartrand, Wohlberg, Vixie, and
  Bollt]{chartrand2009gradient}
R.~Chartrand, B.~Wohlberg, K.~Vixie, and E.~Bollt.
\newblock A gradient descent solution to the {M}onge-{K}antorovich problem.
\newblock \emph{Applied Mathematical Sciences}, 3\penalty0 (22):\penalty0
  1071--1080, 2009.

\bibitem[Chewi et~al.(2020)Chewi, Le~Gouic, Lu, Maunu, and
  Rigollet]{chewi2020svgd}
S.~Chewi, T.~Le~Gouic, C.~Lu, T.~Maunu, and P.~Rigollet.
\newblock Svgd as a kernelized {W}asserstein gradient flow of the chi-squared
  divergence.
\newblock \emph{Advances in Neural Information Processing Systems},
  33:\penalty0 2098--2109, 2020.

\bibitem[Cranmer et~al.(2020)Cranmer, Brehmer, and Louppe]{cranmer2020frontier}
K.~Cranmer, J.~Brehmer, and G.~Louppe.
\newblock The frontier of simulation-based inference.
\newblock \emph{Proceedings of the National Academy of Sciences}, 117\penalty0
  (48):\penalty0 30055--30062, 2020.

\bibitem[Cuturi(2013)]{cuturi2013sinkhorn}
M.~Cuturi.
\newblock Sinkhorn distances: Lightspeed computation of optimal transport.
\newblock \emph{Advances in neural information processing systems}, 26, 2013.

\bibitem[Figalli and Glaudo(2021)]{figalli2021invitation}
A.~Figalli and F.~Glaudo.
\newblock \emph{An Invitation to Optimal Transport, {W}asserstein Distances,
  and Gradient Flows}.
\newblock EMS Press, 2021.
\newblock \doi{10.4171/ETB/22}.

\bibitem[Gangbo(1994)]{gangbo1994elementary}
W.~Gangbo.
\newblock An elementary proof of the polar factorization of vector-valued
  functions.
\newblock \emph{Archive for rational mechanics and analysis}, 128:\penalty0
  381--399, 1994.

\bibitem[Grathwohl et~al.(2018)Grathwohl, Chen, Bettencourt, Sutskever, and
  Duvenaud]{grathwohl2018ffjord}
W.~Grathwohl, R.~T. Chen, J.~Bettencourt, I.~Sutskever, and D.~Duvenaud.
\newblock {FFJORD}: Free-form continuous dynamics for scalable reversible
  generative models.
\newblock \emph{arXiv preprint arXiv:1810.01367}, 2018.

\bibitem[Greenberg et~al.(2019)Greenberg, Nonnenmacher, and
  Macke]{greenberg2019automatic}
D.~Greenberg, M.~Nonnenmacher, and J.~Macke.
\newblock Automatic posterior transformation for likelihood-free inference.
\newblock In \emph{International Conference on Machine Learning}, pages
  2404--2414. PMLR, 2019.

\bibitem[Gu et~al.(2022)Gu, Birmpa, Pantazis, Rey-Bellet, and
  Katsoulakis]{gu2022lipschitz}
H.~Gu, P.~Birmpa, Y.~Pantazis, L.~Rey-Bellet, and M.~A. Katsoulakis.
\newblock Lipschitz regularized gradient flows and latent generative particles.
\newblock \emph{arXiv preprint arXiv:2210.17230}, 2022.

\bibitem[Knothe(1957)]{knothe1957contributions}
H.~Knothe.
\newblock Contributions to the theory of convex bodies.
\newblock \emph{Michigan Mathematical Journal}, 4\penalty0 (1):\penalty0
  39--52, 1957.

\bibitem[Kobyzev et~al.(2020)Kobyzev, Prince, and
  Brubaker]{kobyzev2020normalizing}
I.~Kobyzev, S.~J. Prince, and M.~A. Brubaker.
\newblock Normalizing flows: {A}n introduction and review of current methods.
\newblock \emph{IEEE transactions on pattern analysis and machine
  intelligence}, 43\penalty0 (11):\penalty0 3964--3979, 2020.

\bibitem[Liu et~al.(2021)Liu, Zhou, Jiao, and Huang]{liu2021wasserstein}
S.~Liu, X.~Zhou, Y.~Jiao, and J.~Huang.
\newblock Wasserstein generative learning of conditional distribution.
\newblock \emph{arXiv preprint arXiv:2112.10039}, 2021.

\bibitem[Lueckmann et~al.(2019)Lueckmann, Bassetto, Karaletsos, and
  Macke]{lueckmann2019likelihood}
J.-M. Lueckmann, G.~Bassetto, T.~Karaletsos, and J.~H. Macke.
\newblock Likelihood-free inference with emulator networks.
\newblock In \emph{Symposium on Advances in Approximate Bayesian Inference},
  pages 32--53. PMLR, 2019.

\bibitem[Makkuva et~al.(2020)Makkuva, Taghvaei, Oh, and
  Lee]{makkuva2020optimal}
A.~Makkuva, A.~Taghvaei, S.~Oh, and J.~Lee.
\newblock Optimal transport mapping via input convex neural networks.
\newblock In \emph{International Conference on Machine Learning}, pages
  6672--6681. PMLR, 2020.

\bibitem[Marzouk et~al.(2016)Marzouk, Moselhy, Parno, and
  Spantini]{marzouk2016introduction}
Y.~Marzouk, T.~Moselhy, M.~Parno, and A.~Spantini.
\newblock Sampling via measure transport: An introduction.
\newblock In \emph{Handbook of Uncertainty Quantification}, pages 1--41.
  Springer International Publishing, Cham, 2016.
\newblock ISBN 978-3-319-11259-6.
\newblock \doi{10.1007/978-3-319-11259-6_23-1}.

\bibitem[Mirza and Osindero(2014)]{mirza2014conditional}
M.~Mirza and S.~Osindero.
\newblock Conditional generative adversarial nets.
\newblock \emph{arXiv preprint arXiv:1411.1784}, 2014.

\bibitem[Nadaraya(1964)]{nadaraya1964estimating}
E.~A. Nadaraya.
\newblock On estimating regression.
\newblock \emph{Theory of Probability \& Its Applications}, 9\penalty0
  (1):\penalty0 141--142, 1964.

\bibitem[Nott et~al.(2018)Nott, Ong, Fan, and Sisson]{nott2018high}
D.~J. Nott, V.~M.-H. Ong, Y.~Fan, and S.~Sisson.
\newblock High-dimensional {ABC}.
\newblock In \emph{Handbook of Approximate Bayesian Computation}, pages
  211--241. Chapman and Hall/CRC, 2018.

\bibitem[Onken et~al.(2021)Onken, Fung, Li, and Ruthotto]{onken2021ot}
D.~Onken, S.~W. Fung, X.~Li, and L.~Ruthotto.
\newblock {OT}-flow: Fast and accurate continuous normalizing flows via optimal
  transport.
\newblock In \emph{Proceedings of the AAAI Conference on Artificial
  Intelligence}, volume 35(10), pages 9223--9232, 2021.

\bibitem[Papamakarios et~al.(2021)Papamakarios, Nalisnick, Rezende, Mohamed,
  and Lakshminarayanan]{papamakarios2021normalizing}
G.~Papamakarios, E.~Nalisnick, D.~J. Rezende, S.~Mohamed, and
  B.~Lakshminarayanan.
\newblock Normalizing flows for probabilistic modeling and inference.
\newblock \emph{The Journal of Machine Learning Research}, 22\penalty0
  (1):\penalty0 2617--2680, 2021.

\bibitem[Peyr{\'e} et~al.(2019)Peyr{\'e}, Cuturi,
  et~al.]{peyre2019computational}
G.~Peyr{\'e}, M.~Cuturi, et~al.
\newblock Computational optimal transport: With applications to data science.
\newblock \emph{Foundations and Trends{\textregistered} in Machine Learning},
  11\penalty0 (5-6):\penalty0 355--607, 2019.

\bibitem[Pooladian and Niles-Weed(2021)]{pooladian2021entropic}
A.-A. Pooladian and J.~Niles-Weed.
\newblock Entropic estimation of optimal transport maps.
\newblock \emph{arXiv preprint arXiv:2109.12004}, 2021.

\bibitem[Rezende and Mohamed(2015)]{rezende2015variational}
D.~Rezende and S.~Mohamed.
\newblock Variational inference with normalizing flows.
\newblock In \emph{International conference on machine learning}, pages
  1530--1538. PMLR, 2015.

\bibitem[Robert et~al.(1999)Robert, Casella, and Casella]{robert1999monte}
C.~P. Robert, G.~Casella, and G.~Casella.
\newblock \emph{Monte {C}arlo statistical methods}, volume~2.
\newblock Springer, 1999.

\bibitem[Rosenblatt(1952)]{rosenblatt1952remarks}
M.~Rosenblatt.
\newblock Remarks on a multivariate transformation.
\newblock \emph{The annals of mathematical statistics}, 23\penalty0
  (3):\penalty0 470--472, 1952.

\bibitem[Ruthotto and Haber(2021)]{ruthotto2021introduction}
L.~Ruthotto and E.~Haber.
\newblock An introduction to deep generative modeling.
\newblock \emph{GAMM-Mitteilungen}, 44\penalty0 (2):\penalty0 e202100008, 2021.

\bibitem[Saharia et~al.(2022)Saharia, Ho, Chan, Salimans, Fleet, and
  Norouzi]{saharia2022image}
C.~Saharia, J.~Ho, W.~Chan, T.~Salimans, D.~J. Fleet, and M.~Norouzi.
\newblock Image super-resolution via iterative refinement.
\newblock \emph{IEEE Transactions on Pattern Analysis and Machine
  Intelligence}, 2022.

\bibitem[Santambrogio(2015)]{santambrogio2015optimal}
F.~Santambrogio.
\newblock Optimal transport for applied mathematicians.
\newblock \emph{Birk{\"a}user, NY}, 55\penalty0 (58-63):\penalty0 94, 2015.

\bibitem[Sisson et~al.(2018)Sisson, Fan, and Beaumont]{sisson2018handbook}
S.~A. Sisson, Y.~Fan, and M.~Beaumont.
\newblock \emph{Handbook of approximate {B}ayesian computation}.
\newblock CRC Press, 2018.

\bibitem[Spantini et~al.(2022)Spantini, Baptista, and
  Marzouk]{spantini2022coupling}
A.~Spantini, R.~Baptista, and Y.~Marzouk.
\newblock Coupling techniques for nonlinear ensemble filtering.
\newblock \emph{SIAM Review}, 64\penalty0 (4):\penalty0 921--953, 2022.

\bibitem[Tabak and Turner(2013)]{tabak2013family}
E.~G. Tabak and C.~V. Turner.
\newblock A family of nonparametric density estimation algorithms.
\newblock \emph{Communications on Pure and Applied Mathematics}, 66\penalty0
  (2):\penalty0 145--164, 2013.

\bibitem[Tabak and Vanden-Eijnden(2010)]{tabak2010density}
E.~G. Tabak and E.~Vanden-Eijnden.
\newblock Density estimation by dual ascent of the log-likelihood.
\newblock \emph{Communications in Mathematical Sciences}, 8\penalty0
  (1):\penalty0 217--233, 2010.

\bibitem[Taghvaei and Hosseini(2022)]{taghvaei2022optimal}
A.~Taghvaei and B.~Hosseini.
\newblock An optimal transport formulation of {B}ayes’ law for nonlinear
  filtering algorithms.
\newblock In \emph{2022 IEEE 61st Conference on Decision and Control (CDC)},
  pages 6608--6613. IEEE, 2022.

\bibitem[Trigila and Tabak(2016)]{trigila2016data}
G.~Trigila and E.~G. Tabak.
\newblock Data-driven optimal transport.
\newblock \emph{Communications on Pure and Applied Mathematics}, 69\penalty0
  (4):\penalty0 613--648, 2016.

\bibitem[Trippe and Turner(2017)]{trippe2018conditional}
B.~L. Trippe and R.~E. Turner.
\newblock Conditional density estimation with {B}ayesian normalising flows.
\newblock In \emph{Second workshop on Bayesian Deep Learning}, 2017.

\bibitem[Uscidda and Cuturi(2023)]{uscidda2023monge}
T.~Uscidda and M.~Cuturi.
\newblock The monge gap: A regularizer to learn all transport maps.
\newblock \emph{arXiv preprint arXiv:2302.04953}, 2023.

\bibitem[Villani et~al.(2009)]{villani2009optimal}
C.~Villani et~al.
\newblock \emph{Optimal transport: old and new}, volume 338.
\newblock Springer, 2009.

\bibitem[Watson(1964)]{watson1964smooth}
G.~S. Watson.
\newblock Smooth regression analysis.
\newblock \emph{Sankhy{\=a}: The Indian Journal of Statistics, Series A}, pages
  359--372, 1964.

\bibitem[Wehenkel and Louppe(2019)]{wehenkel2019unconstrained}
A.~Wehenkel and G.~Louppe.
\newblock Unconstrained monotonic neural networks.
\newblock \emph{Advances in neural information processing systems}, 32, 2019.

\bibitem[Winkler et~al.(2019)Winkler, Worrall, Hoogeboom, and
  Welling]{winkler2019learning}
C.~Winkler, D.~Worrall, E.~Hoogeboom, and M.~Welling.
\newblock Learning likelihoods with conditional normalizing flows.
\newblock \emph{arXiv preprint arXiv:1912.00042}, 2019.

\end{thebibliography}

\appendix
\section{Map parametrization and simulation details} \label{sec:parameterization}

In this section we discuss our parameterization for the elementary potential functions $\varphi_\beta(z) = \sum_{j} \beta_j F_j(z)$ that are used in the numerical experiments of Section~\ref{sec:numerical_example}. 

In this work we selected the features $(F_j)$ to be inverse multiquadric kernels or radially-symmetric kernels of the form
\begin{equation}
    F(r)= r\erf\left(\frac{r}{\alpha}\right) +\frac{\alpha e^{-(r/\alpha)^2}}{\sqrt{\pi}}
\end{equation}
where $\alpha \in \R_{>0}$ is the bandwidth and $r=\|z-z_{c}\|$ is the radius for some center point $z_{c} \in \R^n$. This choice aligns with the approach presented in the first modern version of normalizing flows~\citep{tabak2013family}, in which the features apply local expansions or contractions of the sample points around the centers $z_{c}$. 

Once the functional form of the kernels is prescribed, the elementary potential function is completely defined by the choice for the bandwidths $\alpha$ and the centers $z_{c}$. In our numerical experiments, we selected the centers uniformly at random from the  samples $z_t^i \sim \rho_t$ of the reference distribution and the samples $(z')^i \sim \mu$ of the joint (target) distribution. For problems with high-dimensional parameters and observations, such as the Lotka-Volterra example in Section~\ref{sec:LotkaVolterra}, we found that adapting the random sampling for the centers improves the speed of convergence of Algorithm~\ref{alg:mf}. %
In particular, we selected the observation location $y_c$ of the center points $z_c = (y_c,x_c)$ to be near the particular observation of interest, $y^{*}$, more frequently. %
This choice refines the map pushing forward $\rho_t(y,x)$ to $\mu(y,x)$ around $y^{*}$, which is the map used to sample the target conditional $\mu(x|y^*)$.

Lastly, the bandwidth $\alpha$ is chosen according to the rule of thumb described in \citet{trigila2016data}. That is,
\begin{equation} \label{eq:bandwidth}
    \alpha = \left(n_{p}\left(\frac{1}{\tilde{\rho}(z_{c})}+\frac{1}{\tilde{\mu}(z_{c})}\right) \right)^{1/d}
\end{equation}
where $\tilde{\rho}$ and $\tilde{\mu}$ are kernel density estimates (KDE) of the reference and the target distributions, respectively. The rationale behind~\eqref{eq:bandwidth} is to have kernels with a larger bandwidth where there are fewer samples of the reference and the target distributions, and with a smaller bandwidth that can finely resolve the density in regions of the domain where the distributions are more concentrated. Given that the kernel estimator is only needed to compute the scalar bandwidth, they are not meant to be very accurate and updated at every step of the algorithm. %
In this work, the target density is time independent and hence its KDE is computed only once at the beginning of the procedure. The KDE of the reference distribution is instead updated after every 200 steps of the algorithm. The scalar $n_{p} \in \R_{>0}$ is a problem dependent parameter that can be either set to a fixed value (e.g., $n_p = 0.01$ in our experiments) or selected via cross-validation.

To reduce the impact of the specific bandwidth adopted in our procedure, we further multiplied the value of $\alpha$ in~\eqref{eq:bandwidth} by a time dependent constant $m(t) \in \R_{>0}$, which decreases as the algorithm advances. In particular, at the beginning of the experiment, the radius of influence of the kernels (i.e., features that result in local expansions or contractions) are set to be large in order to cover the entire domain containing the samples of $\mu(x)\mu(y)$. As the simulation advances, we gradually decrease the value of $m(t)$ to produce a more localized action of the elementary maps. In our experiments we chose
$$
m(t)=1+\frac{m_{0}}{1+e^{(t-t_{\text{max}})/\sigma}}
$$
with the parameters taken to be $m_{0}=10$, and $\sigma=t_{\text{max}}/10$, where $t_{\text{max}}$ is the maximum number of steps we allow the algorithm to complete.

\section{Derivation of the Jacobian and Hessian} \label{sec:GradientHessian}

In this section we derive expressions for the Jacobian and Hessian of the empirical objective functional $\widehat{\mathcal{J}}_t$ in~\eqref{eq:emp_objective} with respect to the parameters $(\beta_j)$ of the elementary potential function $\varphi_\beta$. To compute these derivative, in Proposition~\ref{prop:asymptotic_expansion} we first derive an expression for the c-transform of the parametric potential function appearing inside the objective functional.

\begin{proposition} \label{prop:asymptotic_expansion} For the cost function $c_{\lambda}(z,z') = \frac{1}{2}(\lambda \|y - y'\|^2 + \|x - x'\|^2),$ let $\varphi_\beta^c(z') = \min_{z} \{c_{\lambda}(z,z')-\varphi_{\beta}(z)\}$ be the c-transform of the differentiable function $\varphi_\beta(z)\colon \R^{n} \rightarrow \R$ where $\varphi_\beta(z) = \sum_j \beta_j F_j(z)$. Then, $\varphi_\beta^c$ has an second-order asymptotic expansion in $\beta$ given by
\begin{equation} \label{eq:second_order_expansion}
\varphi_\beta^c(z') = -\sum_j \beta_j F_j(z') - \frac{1}{2}\sum_{j,k} \beta_j\beta_k\langle \nabla F_j(z') ,\nabla_{\lambda} F_k(z') \rangle + \mathcal{O}(\|\beta\|^3).
\end{equation} %
\end{proposition}

\begin{proof}
Let $\bar{z}$ be the optimal $z$ that  attains the minimum value for the c-transform, i.e., $\varphi_\beta^c(z') = c_\lambda(\bar{z},z') - \varphi_\beta(\bar{z})$. 
For a differentiable function $\varphi_\beta$, the optimal value $\bar{z}$ satisfies $\nabla_{z}c_{\lambda}(\bar{z},z')-\nabla_{z}\varphi_\beta(\bar{z})=0$.  
Thus, for the parametric expansion  $\varphi_\beta(z) = \sum_j \beta_j F_j(z)$ %
we have the condition $\bar{z} = z' + \sum_j \beta_j \nabla_\lambda F_j(\bar{z})$, which defines an implicit function for $\bar{z}$ in terms of $z'$. 

Substituting the expression for $\bar{z}$ in the c-transform gives us
\begin{equation} \label{eq:ctransform}
\varphi_\beta^c(z') = \frac{1}{2}\left(\lambda \Big\|\sum_j \beta_j \nabla_y F_j(\bar{z})/\lambda \Big\|^2 + \Big\|\sum_j \beta_j \nabla_x F_j(\bar{z})\Big\|^2 \right) - \sum_j \beta_j F_j(\bar{z}), \\
\end{equation}
where $\bar{z}$ depends on $z'$. A first-order Taylor series expansion of each feature $F_j$ in the first term of~\eqref{eq:ctransform} around $z'$ yields the following second-order asymptotic expansion in the coefficients $\beta$ 
\begin{align*}
    & \lambda \Big\|\sum_j \beta_j \nabla_y F_j(\bar{z})/\lambda \Big\|^2 + \Big\|\sum_j \beta_j \nabla_x F_j(\bar{z})\Big\|^2 \\
    =& \sum_{j,k} \beta_j\beta_k \langle \nabla_y F_j(z'), \nabla_y F_k(z')/\lambda \rangle + \sum_{j,k} \beta_j\beta_k \langle \nabla_x F_j(z'), \nabla_x F_k(z') \rangle + \mathcal{O}(\|\beta\|^3).
\end{align*}
Similarly, a first-order Taylor series expansion of each feature $F_j$ in the second term of~\eqref{eq:ctransform} around $z'$ yields the asymptotic expansion
$$\sum_j \beta_j F_j(\bar{z}) = \sum_j \beta_j F_j(z') + \sum_{j,k} \beta_j\beta_k \langle \nabla F_j(z'), \nabla_\lambda F_k(z') \rangle + \mathcal{O}(\|\beta\|^3).$$
Substituting these expansions in~\eqref{eq:ctransform},  we arrive at the second-order expansion in~\eqref{eq:second_order_expansion} 
after collecting the quadratic terms in $\beta$ and using the definition of the rescaled gradient.
\end{proof}

Using the result of Proposition 1, a second-order asymptotic expansion in $\beta$ for the empirical objective functional is given by
\begin{align*}
    \widehat{\mathcal{J}}_t(\varphi_\beta) =& \frac{1}{N}\sum_{i=1}^N \left(\sum_j \beta_j F_j(z_t^i) \right) - \\
    &\frac{1}{M}\sum_{i=1}^M \left(\sum_j \beta_j F_j((z')^i) + \frac{1}{2}\sum_{j,k} \beta_j\beta_k\langle \nabla F_j((z')^i) ,\nabla_{\lambda} F_k((z')^i) \rangle + \mathcal{O}(\|\beta\|^3) \right).
\end{align*}
Computing the first and second derivatives of the functional above with respect to each coefficient and evaluating the result at $\beta = 0$ results in the Jacobian and Hessian presented in Section~\ref{sec:flows}.

\vspace{0.5cm}

\end{document}